\documentclass[11pt]{article}

\usepackage[final]{acl}
\usepackage{times}
\usepackage{latexsym}

\usepackage[T1]{fontenc}

\usepackage[utf8]{inputenc}

\usepackage{microtype}
\usepackage{amsthm}
\newtheorem{proposition}{Proposition}
\newtheorem{theorem}{Theorem}

\usepackage{subcaption}
\usepackage{algorithm}
\usepackage{algorithmic}
\usepackage{amsthm}
\usepackage{amsmath} 
\usepackage{amssymb}

\usepackage{inconsolata}

\usepackage{graphicx}
\definecolor{delim}{RGB}{20,105,176}
\definecolor{numb}{RGB}{106, 109, 32}
\definecolor{string}{rgb}{0.64,0.08,0.08}
\definecolor{backcolour}{rgb}{0.95,0.95,0.92}

\definecolor{pykeyword}{RGB}{0, 0, 139}     
\definecolor{pystring}{RGB}{165, 42, 42}  
\definecolor{pycomment}{RGB}{85, 107, 47}  
\definecolor{pynumber}{RGB}{105, 105, 105}
\definecolor{darkred}{rgb}{0.5,0,0}
\definecolor{1}{HTML}{E67E22}
\definecolor{2}{HTML}{967aa1}
\definecolor{3}{HTML}{718355}
\definecolor{4}{HTML}{800f2f}
\definecolor{5}{HTML}{4a4e69}

\definecolor{6}{HTML}{2980B9}
\definecolor{7}{HTML}{C0392B}
\definecolor{8}{HTML}{9B59B6}
\usepackage{listings}

\lstdefinelanguage{json}{
    frame=single,
    rulecolor=\color{black},
    showspaces=false,
    showstringspaces=false,
    showtabs=false,
    breaklines=true,
    postbreak=\raisebox{0ex}[0ex][0ex]{\ensuremath{\color{gray}\hookrightarrow\space}},
    breakatwhitespace=true,
    basicstyle=\ttfamily\small,
    upquote=true,
    morestring=[b]",
    stringstyle=\color{string},
    literate=
     *{0}{{{\color{numb}0}}}{1}
}

\usepackage{makecell}

\usepackage{amsmath}
\usepackage{booktabs}
\usepackage{tcolorbox}
\usepackage{tcolorbox}
\tcbuselibrary{breakable}   
\usepackage{siunitx}        
\title{LigBench: A Unified and Human-Aligned Benchmark for LLM-based Research Idea Generation}

\author{
 \textbf{Chenrun Wang\textsuperscript{1,2,*}},
 \textbf{Mingxuan Zhu\textsuperscript{1,*}},
 \textbf{Tiancheng Huang\textsuperscript{1}},
 \textbf{Wenjie Li\textsuperscript{2}},
\\
 \textbf{Yujie Zhang\textsuperscript{2}},
 \textbf{Zichen Zhu\textsuperscript{1}},
 \textbf{Zhiying Zou\textsuperscript{1}},
 \textbf{Kai Yu\textsuperscript{1}},
 \textbf{Lu Chen\textsuperscript{1,2}}$^\dagger$,
\\
\\
 \textsuperscript{1}X-LANCE Lab, School of Computer Science, Shanghai Jiao Tong University, Shanghai,
China
\\
 \textsuperscript{2}Shanghai Innovation Institution, Shanghai, China
\\
\textsuperscript{*}Equal contribution
\qquad
$^\dagger$Corresponding author
}

\begin{document}
\maketitle
\begin{abstract}
With the rapid advancement of large language models (LLMs), research idea generation has attracted increasing attention. Existing approaches enable LLMs to retrieve relevant literature and propose novel ideas for research areas.  However, current evaluation practices for idea generation remain fragmented and lack objective standards, often relying on direct LLM scoring, which limits their ability to provide unified and reliable assessments across a coherent distribution of generated ideas. To address this challenge, we propose LigBench, an automated evaluation benchmark
that enables fine-grained and reliable evaluation of AI     research ideas, consistently applicable across different generation distributions. In addition, we introduce PAIR-IQ, a dataset tailored for training pairwise idea judgment models and serving as an auxiliary reference to support more objective comparative evaluation. Extensive experiments demonstrate that LigBench achieves stable and interpretable evaluations, significantly improving alignment with expert judgments. Furthermore, models trained on PAIR-IQ exhibit enhanced ranking accuracy and robustness, establishing a principled standard for scalable and objective research idea assessment. 

\end{abstract}

\section{Introduction}



The rapid evolution of LLMs~\citep{gpt4, gemini2.5, qwen3} has ushered in a new era of automated research idea generation. Building on the strong reasoning and knowledge synthesis capabilities of these models, a growing body of recent work~\citep{hypothesis, aiscientist, researchagent, manyhead} has proposed LLM-based frameworks for scientific ideation, enabling the generation of novel hypotheses, methodological variations, and promising research directions grounded in large-scale scientific literature. Such capabilities hold great promise for accelerating discovery across disciplines, guiding researchers toward unexplored directions, and reducing the time and cost associated with manual ideation. However, as the utility of these systems continues to grow, the lack of rigorous, objective, and reproducible evaluation methodologies remains a critical bottleneck for the systematic assessment of generated ideas.


Current evaluation strategies for idea generation are often fragmented. Some studies~\citep{zheng2023judging, nova, canllm} rely on direct scoring by LLMs, where models assign quality scores to their own or peer-generated ideas, while others~\citep{scipip, nova} depend on human experts to perform pairwise or ranking-based evaluations. Automated metrics such as novelty, diversity, or semantic similarity offer partial insights but fail to capture the multi-dimensional nature of research idea quality. Overall, these approaches lack consistency: they may be sensitive to prompt phrasing, suffer from inter-annotator variability, or provide only relative assessments. Consequently, there is no unified framework to benchmark idea generation systems comprehensively.

Pairwise comparison has emerged as a promising approach to sidestep some subjective biases inherent in absolute scoring. In this paradigm, two candidate ideas are presented side by side, and an evaluator (whether human or automated) selects the one deemed superior along specified criteria. Aggregating multiple pairwise comparisons can yield a ranking of ideas with higher consistency than solitary ratings. Nevertheless, pairwise evaluation alone does not constitute a full benchmarking solution: it requires a sufficiently large and diverse pool of test pairs, clear annotation protocols, and scalable tooling to collect judgments at scale.

To address these challenges, we introduce LigBench, an automated evaluation framework tailored for research idea assessment. LigBench operationalizes multi-dimensional quality criteria by orchestrating large-scale pairwise comparisons, metadata tracking, statistical aggregation, and automatic score updates into a cohesive pipeline. We drew on the thinking process experts use when evaluating an idea. Experts typically begin by searching for existing results related to the idea for reference or comparison. In addition, pairwise comparison alone is not sufficient to objectively and convincingly reflect the quality of an idea. Therefore, we have designed and provided a complete process based on pairwise comparison combined with the Elo mechanism to continuously update scores. We also present the PAIR-IQ dataset, a curated dataset of research papers' ideas containing structured idea representations. PAIR-IQ serves both as a gold-standard reference for comparative evaluation and as a training resource for models designed to predict pairwise judgments. 

Our main contributions are summarized as follows:

\begin{itemize}
    \item \textbf{Unified Evaluation Framework:} We introduce LigBench as a unified and extensible evaluation framework that systematically integrates data curation, idea retrieval, pairwise comparison, score propagation, and result aggregation within a single automated pipeline.
    \item \textbf{PAIR-IQ Dataset:} We release PAIR-IQ dataset, comprising over 11000 conference papers with debiased evaluation scores and formalized representations to support objective assessment of research ideas.
    \item \textbf{Open-Source Resources:} We will release the PAIR-IQ dataset together with the complete LigBench evaluation pipeline to facilitate reproducibility and enable community-driven extensions. The PAIR-IQ dataset is publicly available at \url{https://huggingface.co/datasets/USER3IjEBHj9/PAIR-IQ/tree/main}. The full LigBench evaluation pipeline will be released in the near future. 
    \item \textbf{Benchmarking and Analysis:} We conduct extensive experiments with state-of-the-art LLMs, demonstrating that LigBench can discern subtle differences in idea quality and offering insights into model strengths and weaknesses.
\end{itemize}

\section{Related Work}


Current LLM-based ideation systems typically employ evaluation protocols that are tightly coupled to their own generation pipelines, reflecting framework-specific assumptions rather than a unified or standardized assessment methodology. Existing approaches primarily rely on either direct scoring by large language models or judgments from human experts, and are often designed to operate only within the context of a particular ideation system. As a result, these evaluation methods lack the flexibility and generality required for consistent assessment across different models, prompting strategies, and idea distributions. For instance, SciPIP~\citep{scipip} relies on evaluations conducted by human researchers, limiting scalability and reproducibility. AI Idea Bench~\citep{ideabench} introduces an evaluation protocol that is heavily dependent on the underlying idea generation framework and does not support standalone assessment of individual ideas. CoI~\citep{coi} proposes a pairwise comparison approach, but restricts evaluation to comparisons among generated ideas, making objective comparison with human-authored research ideas challenging. Consequently, as highlighted by a recent survey~\citep{survey}, evaluation remains the most critical bottleneck in LLM-based scientific ideation.

\section{Method}
To provide a unified, objective, and human-aligned evaluation of research ideas, we propose \textbf{LigBench}, an automated benchmark designed for systematic idea assessment. As illustrated in Figure~\ref{fig:ligbench_pipeline}, the evaluation of each idea is decomposed into four aspects: \textit{rating}, \textit{contribution}, \textit{soundness}, and \textit{novelty}, with each aspect scored on a \textbf{0-5 scale}. Here, \textit{rating} represents the overall quality as perceived by reviewers, \textit{contribution} reflects the significance of the idea within its research domain, \textit{soundness} captures methodological rigor and validity, and \textit{novelty} measures the originality of the proposed idea.

\begin{figure*}[t]
  \centering
  \includegraphics[width=\textwidth,trim=0 0 0 0,clip]{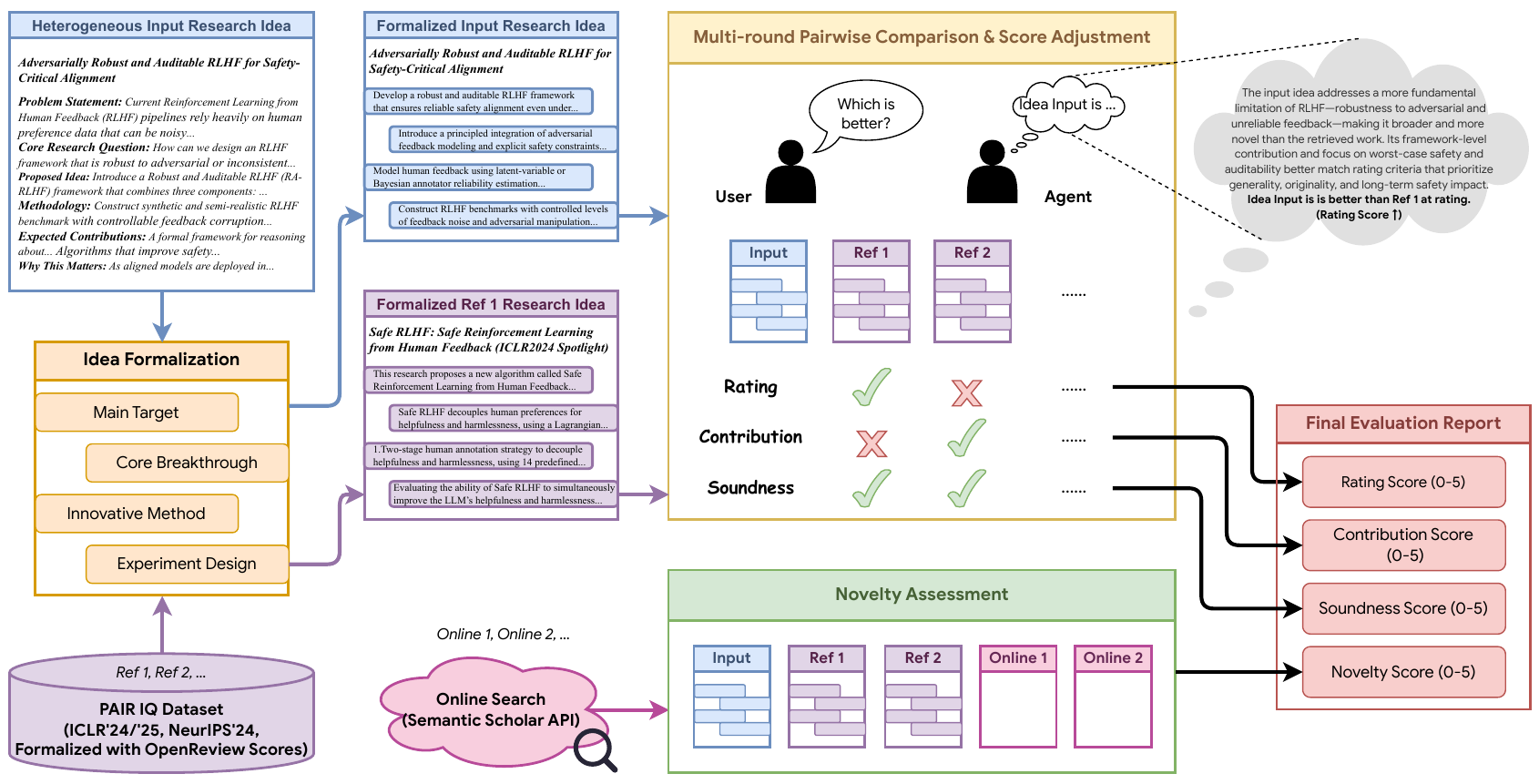} \hfill
  \caption {Overview of the LigBench evaluation pipeline. Given a target idea, LigBench performs idea formalization, retrieves related papers, conducts LLM-based pairwise comparisons, and iteratively updates scores using PAIR-IQ until convergence across multiple evaluation dimensions.}
  \label{fig:ligbench_pipeline}
\end{figure*}

Following the pipeline shown in Figure~\ref{fig:ligbench_pipeline}, each idea is first transformed into a formalized representation and then iteratively compared against similar papers retrieved from the database using a large language model. In each comparison, the model evaluates the relative quality between idea pairs, and the resulting judgments are aggregated with the objective scores from the PAIR-IQ dataset to update the score of the target idea. This process is repeated until convergence, producing a final score that is consistent across different sources and aligned with human preferences.

\subsection{Idea Formalization}
To begin with, a fundamental challenge in building a unified idea evaluation framework is the heterogeneity in the formats of the ideas under evaluation. Due to inherent differences in the distributions of ideas generated by different models and frameworks, directly comparing these ideas is prone to bias. For instance, when using large language models as evaluators, they often tend to favor ideas that are described in greater detail, potentially overlooking those that are more concise yet fundamentally more innovative or valuable. Moreover, different studies are conducted under diverse frameworks, leading to substantial variations in how final ideas are presented, including their structure and formatting~\citep{chiang-lee-2023-large,liu-etal-2023-g}, which further complicates fair and consistent comparison across models and methods.

To mitigate these issues, it is necessary to standardize ideas across different distributions. Here, we focus on the core aspects and boundaries of each idea and decouple it into four distinct components:
\begin{itemize}
    \item \textbf{Main Target}: a concise, one-sentence summary of the task and objective that the idea aims to address.
    \item \textbf{Core Breakthrough}: the key advancement or novel contribution introduced by the idea compared to prior work.
    \item \textbf{Innovative Methods}: the concrete methodological approaches or techniques proposed to achieve the target.
    \item \textbf{Experimental Design}: the experimental setup and evaluation plan designed to validate the effectiveness of the proposed methods.
\end{itemize}

Formally, let $I \in \mathcal{I}$ denote an input idea from the idea space. We define an LLM-based decomposition operator $\mathcal{D}_{\text{LLM}} : \mathcal{I} \rightarrow \mathcal{T} \times \mathcal{B} \times \mathcal{M} \times \mathcal{E}$ such that
\begin{equation}
\label{eq: formalize}
(T, B, M, E) = \mathcal{D}_{\text{LLM}}(I),
\end{equation}
where $T \in \mathcal{T}$ represents the main target, $B \in \mathcal{B}$ denotes the core breakthrough, $M \in \mathcal{M}$ corresponds to the innovative methods, and $E \in \mathcal{E}$ specifies the experimental design. This formulation casts idea formalization as a structured, LLM-driven mapping into a unified representation space, facilitating systematic comparison and subsequent quantitative evaluation. The concrete prompting strategy used for this formalization is provided in Appendix~\ref{app: formalization}.

By explicitly disentangling complementary semantic facets of an idea into structured components, this LLM-driven decomposition yields a unified yet fine-grained representation that captures the substantive content while abstracting away superficial variations in expression and format. Consequently, it mitigates evaluation bias arising from heterogeneous idea sources and presentation styles, enabling more objective and consistent comparison across models and frameworks. Furthermore, this structured representation establishes a principled foundation for subsequent quantitative analyses, including assessments of soundness, novelty, contribution, and other relevant criteria.

\subsection{PAIR-IQ Construct}

To better evaluate research ideas according to human preferences, we collected and constructed the \textbf{PAIR-IQ dataset}, comprising over 11,000 papers from \textbf{ICLR 2024, ICLR 2025, and NeurIPS 2024}, including oral, spotlight, poster, and rejected papers. For each paper, we retrieved \textbf{ratings, contribution scores, and soundness scores from OpenReview}, and projected each score to a standardized 0–5 range to ensure comparability across different venues and scoring scales.

Figure~\ref{fig:dataset_statistics} illustrates the composition of the PAIR-IQ dataset. As shown in Figure~\ref{fig:category_dist}, the dataset draws from three major machine learning venues, with ICLR 2025 contributing the largest proportion (36.7\%), followed by ICLR 2024 (32.4\%) and NeurIPS 2024 (30.9\%). Figure~\ref{fig:conf_dist} presents the distribution across acceptance categories: poster papers constitute the majority (57.3\%), followed by rejected submissions (36.6\%), spotlight papers (4.3\%), and oral presentations (1.8\%). This distribution ensures that the dataset captures a wide spectrum of paper quality levels, from top-tier accepted works to borderline and rejected submissions. Figure~\ref{fig:topic_dist} further depicts the topical diversity of the dataset, spanning 12 research areas. Large language models (LLM) represent the most prevalent topic (19.3\%), followed by training methodologies (15.4\%) and reinforcement learning (13.4\%), reflecting the current research trends in the machine learning community. Detailed information and a specific example of the dataset are provided in Appendix~\ref{app:PAIR-IQ extra information}.

\begin{figure*}[t]
    \centering
    \begin{subfigure}[b]{0.45\textwidth}
        \centering
        \includegraphics[width=\textwidth]{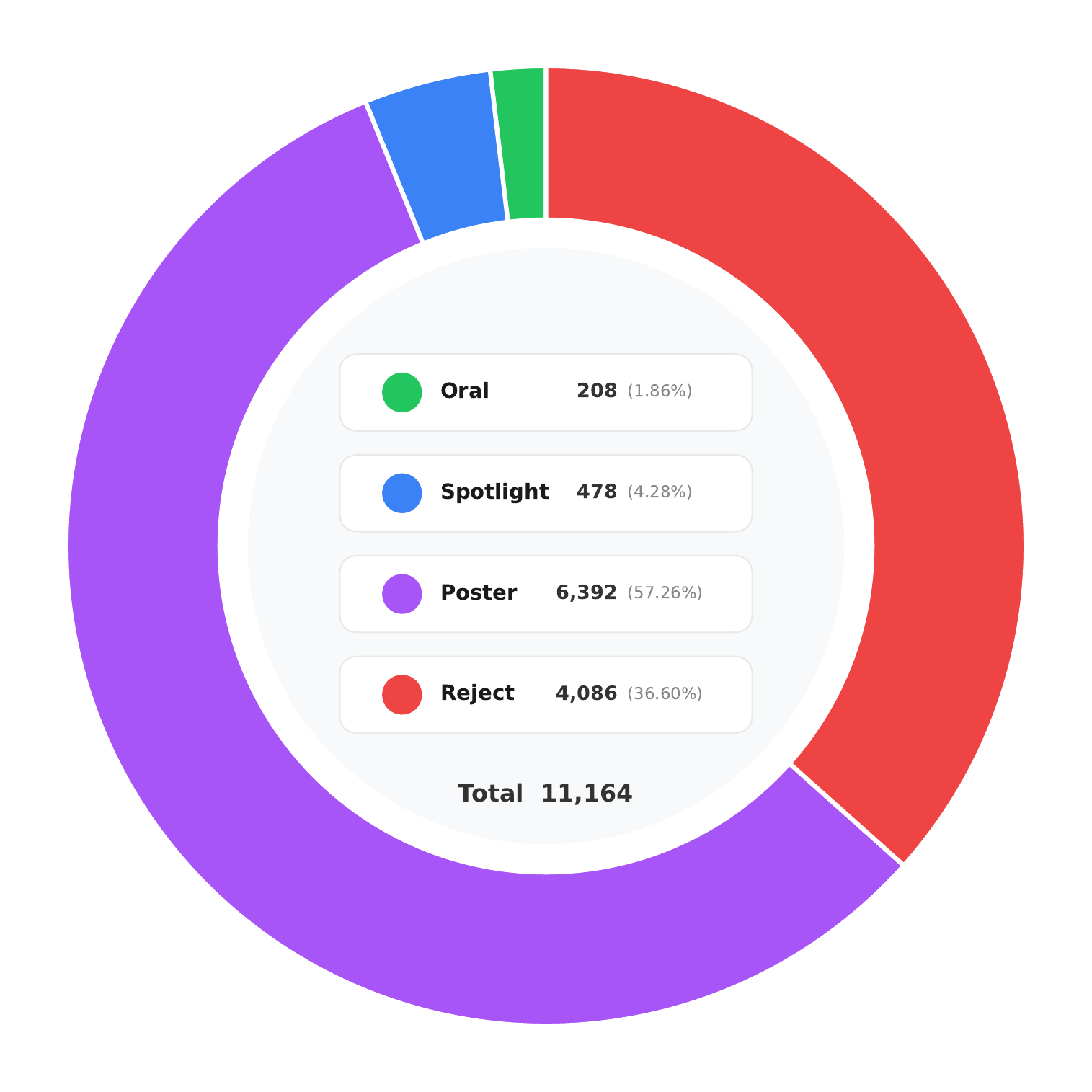}
        \caption{Distribution by acceptance category.}
        \label{fig:conf_dist}
    \end{subfigure}
    \hfill
    \begin{subfigure}[b]{0.45\textwidth}
        \centering
        \includegraphics[width=\textwidth]{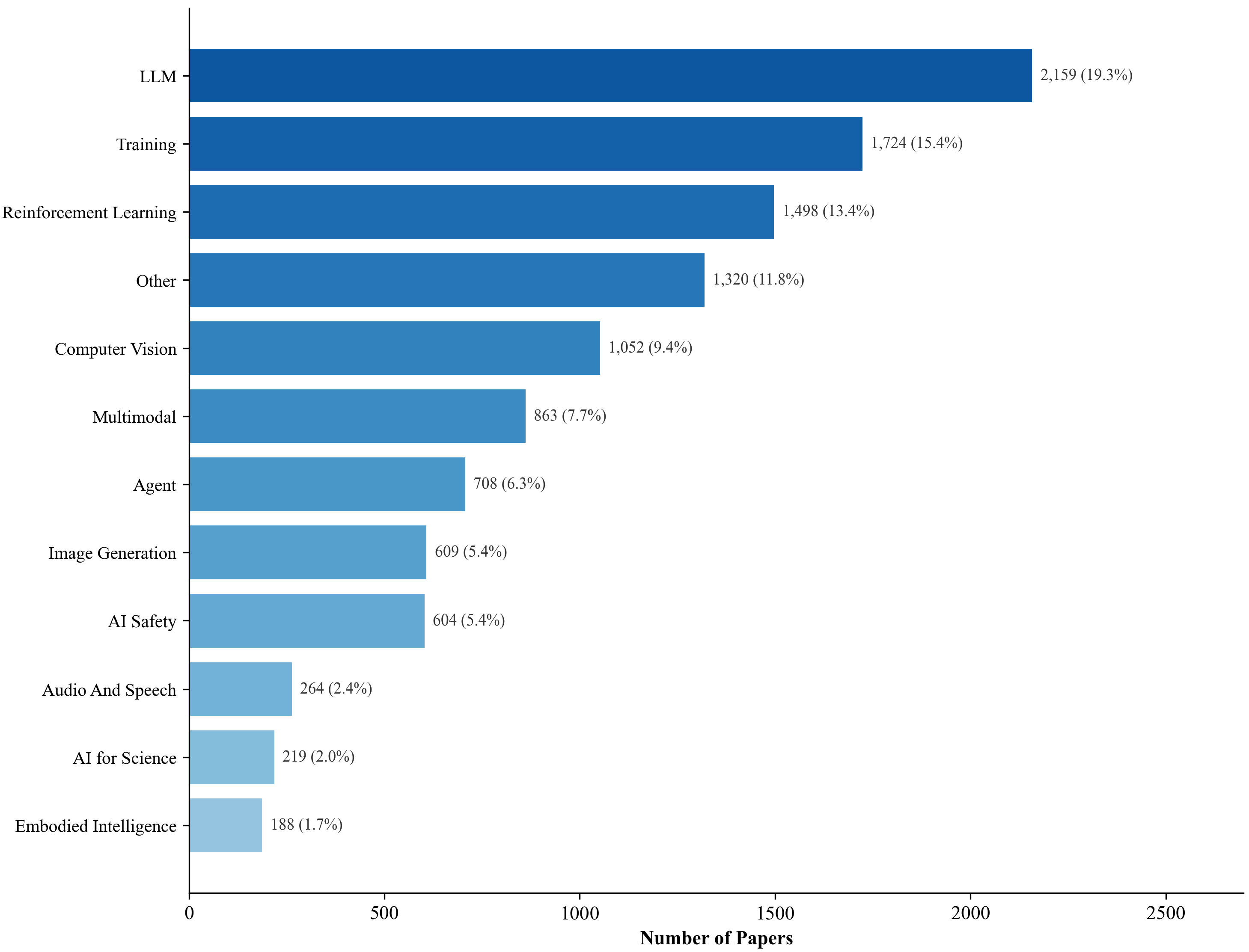}
        \caption{Distribution by research topic.}
        \label{fig:topic_dist}
    \end{subfigure}
    \caption{Statistical overview of the PAIR-IQ dataset. The dataset comprises 11,164 papers collected from three major machine learning conferences (ICLR 2024, ICLR 2025, and NeurIPS 2024), spanning diverse acceptance categories and research topics. }
    \label{fig:dataset_statistics}
\end{figure*}

Following the idea formalization framework introduced in the previous subsection, we now describe in detail how we extract and formalize research ideas from full academic papers. For each paper, we parsed all textual content and employed a LLM to extract structured components, including Main Target, Core Breakthrough, Innovative Methods, and Experimental Design. These components provide a standardized representation of each research idea, enabling consistent downstream evaluation and facilitating subsequent scoring procedures. A more detailed description of the prompting strategy and extraction process is provided in Appendix~\ref{app:paper_formalization}.

Next, we address potential biases in the collected ratings. Since different venues and years may exhibit systematic differences in reviewer behavior, the resulting score distributions can vary significantly across the dataset. To ensure that the scores of all papers are comparable and the overall distribution is consistent, we perform \textbf{score debiasing} using a mean-shifting procedure.

\begin{table}
 \centering
\resizebox{0.48\textwidth}{!}{
 
  \begin{tabular}{lccc}
    
\toprule
     & \textbf{Rating} & \textbf{Contribution} & \textbf{Soundness} \\
    \midrule
    ICLR 2024     & 2.740 & 2.689 & 2.472          \\
    ICLR 2025    & 2.821 & 2.687 & 2.498          \\
    NeurIPS 2024     & 2.918 & 2.914 & 2.721           \\
    all     & 2.825 & 2.758 & 2.559           \\
    
\bottomrule
  \end{tabular}
 }
  \caption{Average scores of papers from different venues in terms of rating, contribution, and soundness. The values highlight systematic differences across conferences and years, motivating the need for score debiasing. }
  \label{tab:debia}

\end{table}

Specifically, let $s_{i}^{(v)}$ denote the original score (rating, contribution, or soundness) of paper $i$ from venue $v$, with venue mean $\mu_v$ and overall mean $\mu_{\text{all}}$. The debiased score is computed as

\begin{equation}
\label{eq: score_debias}
\hat{s}_{i}^{(v)} = s_{i}^{(v)} - \mu_v + \mu_{\text{all}}.
\end{equation}
This procedure aligns the score distributions across different venues and years, ensuring that subsequent evaluations reflect the intrinsic quality of each idea rather than systematic biases in reviewer ratings. Additionally, the dataset was further processed through embedding computation to facilitate downstream evaluation tasks.

\subsection{Pairwise Comparison}
\label{sec:pairwise}

When a new research idea is submitted for evaluation, we first transform it into a formalized representation using the same procedure described in the previous subsection. This ensures that the target idea and reference papers are represented in a unified format.

Based on this representation, we retrieve semantically related papers from the database using a parallel retrieval strategy that combines keyword matching and embedding-based similarity. The target idea is then compared pairwise with each retrieved paper using a LLM, which assesses relative quality along three dimensions: \textit{rating}, \textit{contribution}, and \textit{soundness}. The specific score initialization and the comparison prompt are provided in Appendix~\ref{pair com}.

We update the target idea's scores using an adapted Elo rating algorithm~\citep{elo1978rating}. In standard Elo, the expected probability that idea $A$ wins over idea $B$ is:
\begin{equation}
    E_A = \frac{1}{1 + 10^{(s_B - s_A) / d}},
\end{equation}
where $s_A$ and $s_B$ denote the current scores of ideas $A$ and $B$, and $d$ is a scaling parameter. Given the observed outcome $S_A \in \{0, 0.5, 1\}$, the score is updated as:
\begin{equation}
    s_A^{\text{new}} = \mathcal{C}\left(s_A + K \cdot (S_A - E_A)\right),
\end{equation}
where $K$ is an adaptive update factor and $\mathcal{C}(\cdot)$ is a soft clamping function that constrains scores to $[0, 5]$.

Through multiple rounds of pairwise comparisons, the target idea's scores gradually converge. The evaluation terminates when score changes fall below a predefined threshold, yielding the final assessment. A detailed mathematical analysis of the algorithm, including convergence properties and parameter selection, is provided in Appendix~\ref{appendix:elo}.

\subsection{Novelty Assessment}
\label{sec:novelty}

Through the above procedures, we obtain objective evaluations of research ideas in terms of \textit{rating}, \textit{contribution}, and \textit{soundness}. However, these dimensions do not explicitly capture the \textit{novelty} of an idea. Unlike other evaluation dimensions that can be assessed through comparative analysis, novelty requires determining whether an idea introduces genuinely new insights compared to existing literature.

To address this challenge, we design a novelty assessment module that combines LLM-based initial estimation with similarity-based quantitative analysis. The evaluation proceeds in three stages: (1) obtaining an initial novelty estimate $s_{\text{novelty}}^{(0)}$ from the LLM evaluator; (2) retrieving semantically related papers via the Semantic Scholar API~\citep{kinney2023semantic} and computing a weighted similarity metric $s_{\text{weighted}}$, which is then mapped to a novelty score $s_{\text{sim}}$ through an inverse sigmoid function; and (3) fusing both assessments into the final score:
\begin{equation}
    s_{\text{novelty}}^{\text{final}} = \beta \cdot s_{\text{sim}} + (1-\beta) \cdot s_{\text{novelty}}^{(0)},
\end{equation}
where $\beta=0.7$ assigns higher weight to the similarity-based calculation, reflecting its objectivity, while retaining the LLM's judgment to capture aspects not fully represented by semantic similarity alone. A detailed description of the similarity computation and parameter selection is provided in Appendix~\ref{appendix:novelty}.

\section{Experiments}
\subsection{Model Training}

We also fine-tune models of smaller sizes to investigate whether the scientific "taste" of an area can be learned through knowledge distillation~\citep{hinton2015distilling}. Following Section~\ref{sec:pairwise_evaluation}, we sample \num{15000} thematically similar pairs while ensuring none of the selected papers appear in the test data. To help small models better understand the underlying domain knowledge, we adopt an additional chain-of-thought~\citep{cot} generated by larger models based on ground-truth labels. We choose Qwen2.5-72B-Instruct~\citep{qwen25} as the teacher model, and smaller models from the same family as the student models. As for the training framework, we employ LLaMA-Factory~\citep{llamafactory}, and further details are presented in Appendix~\ref{app:training_hyperparameter}.

\subsection{Pairwise Evaluation}

\label{sec:pairwise_evaluation}

The accuracy of pairwise judgments made by the LLM serves as a key guarantee for the reliability of the entire evaluation framework. To assess this reliability, we conduct a dedicated pairwise evaluation on the three dimensions: \textit{rating}, \textit{contribution}, and \textit{soundness}. The goal of this evaluation is to verify whether the LLM can correctly determine the relative quality between two research ideas under each criterion.

For this purpose, we randomly sample a held-out test set consisting of 269 pairs of thematically similar papers from the database, ensuring that none of the selected papers appear in the training data. Thematic similarity is determined using the same hybrid retrieval strategy based on keyword matching and embedding-based similarity. Each sampled pair is presented to the LLM in its formalized representation, and the model is asked to judge which idea performs better on each of the three dimensions.

The ground-truth labels for pairwise comparison are derived from the \textbf{debiased OpenReview scores}. Specifically, for each paper pair, we compare their debiased scores on \textit{rating}, \textit{contribution}, and \textit{soundness}, and treat the resulting relative ordering as the reference outcome. The accuracy of the LLM’s pairwise judgments is then measured by comparing its predictions against these ground-truth relationships, as summarized in Table~\ref{tab:pairwise_eval}.

\begin{table}
\small
\setlength{\tabcolsep}{5pt}
\renewcommand{\arraystretch}{1.0}
  \centering
  \begin{tabular}{lccc}
    \toprule
    \textbf{Model} & \textbf{Rating} & \textbf{Contrib.} & \textbf{Sound.} \\
    \midrule
    Qwen2.5-7B     & 0.514 & 0.496 & 0.488          \\
    Qwen2.5-14B    & 0.506 & 0.598 &0.541\\
    GPT-4o & 0.549 & 0.615 & 0.747\\
    GPT-4.1     & 0.539 & 0.545 & 0.649         \\
    GPT-5     & \textbf{0.820} & 0.695 & \textbf{0.826}           \\
    GPT-5.2 & 0.801 & 0.707 & 0.803 \\
    Gemini-2.5-pro & 0.706 & 0.607 & 0.750 \\
    Gemini-3-pro & 0.806 & 0.\textbf{713} &  0.814\\
    \midrule
    Qwen2.5-7B(trained)     & 0.714 & 0.710 & 0.712           \\
    Qwen2.5-14B(trained)     & 0.755 & 0.701 & 0.751           \\
    \bottomrule
  \end{tabular}
  \caption{Accuracy of LLM-based pairwise judgments on \textit{rating}, \textit{contribution}, and \textit{soundness}, evaluated against debiased OpenReview ground-truth labels.}
  \label{tab:pairwise_eval}
\end{table}

Overall, stronger models consistently achieve higher accuracy across all three dimensions, indicating that pairwise judgment is a non-trivial capability that benefits from increased model capacity and reasoning ability. Among the off-the-shelf models, GPT-5 substantially outperforms other baselines, achieving accuracies above 0.80 on both \textit{rating} and \textit{soundness}, which suggests a high level of agreement with debiased human judgments.

In contrast, smaller open-source models such as Qwen2.5-7B and Qwen2.5-14B exhibit limited performance, with accuracies close to random guessing in some dimensions, particularly on \textit{rating} and \textit{soundness}. This highlights the difficulty of reliably performing fine-grained comparative evaluation without sufficient model capacity or task-specific supervision.

Notably, models trained on the PAIR-IQ dataset show substantial improvements across all dimensions. Both trained variants achieve consistently higher accuracies than their untrained counterparts, demonstrating that targeted pairwise supervision effectively enhances the model’s ability to align with human evaluation criteria. These results validate the design of the PAIR-IQ dataset and further support the use of pairwise judgment as a reliable component within the LigBench evaluation framework.

Although the pairwise judgment accuracy does not reach near-perfect levels, this outcome also reflects the intrinsic difficulty of the task, as distinguishing fine-grained differences between closely related research ideas is inherently challenging. Nevertheless, when applied within the LigBench framework, even imperfect pairwise judgments can still lead to reliable final evaluations. Starting from a strong base model, repeated score updates over a large number of pairwise comparisons allow the target scores to gradually converge toward reasonable values.

\begin{table}
\small
  \centering
  \begin{tabular}{ccc}
    \toprule
     \textbf{Rating} & \textbf{Contrib.} & \textbf{Sound.} \\
    \midrule
     0.114 & 0.196 & 0.138          \\
    \bottomrule
    
  \end{tabular}
  \caption{Mean debiased score differences for incorrectly judged paper pairs by \textbf{GPT-5} in the pairwise evaluation. The small score gaps indicate that errors have limited impact on iterative score updates.}
  \label{tab:error_analysis}
\end{table}

Further analysis of the error cases, summarized in Table~\ref{tab:error_analysis}, reveals that most incorrect judgments occur when the debiased scores of the two compared papers are very close. This observation suggests that even when the LLM makes an incorrect decision, the resulting update magnitude is limited. According to the score update mechanism described in the previous section, such small score differences lead to correspondingly small update steps, thereby preventing severe error accumulation during the iterative evaluation process.

\subsection{Human Alignment}
Although the pairwise judgment capability of the large language model has been validated using debiased OpenReview scores, it remains essential to perform human verification. This is because even debiased scores cannot guarantee a perfectly objective assessment, and expert judgment is required to ensure alignment with human evaluation standards.

To further verify human alignment, we recruited several PhD-level AI researchers to conduct pairwise comparisons on \textbf{100 randomly sampled idea pairs}. For each pair, experts selected the superior idea along the three dimensions: \textit{rating}, \textit{contribution}, and \textit{soundness}. The resulting human judgments were then compared with the LLM’s pairwise predictions. The alignment results are summarized in Table~\ref{tab:human_alignment}, indicating a strong consistency between the LLM and expert evaluations.

   
    

\begin{table}[t]
\centering

\begin{minipage}{0.45\textwidth}
\centering
\small
\begin{tabular}{ccc}
\toprule
\textbf{Rating} & \textbf{Contrib.} & \textbf{Sound.} \\
\midrule
71\% & 79\% & 73\% \\
\bottomrule
\end{tabular}
\caption{Agreement between LLM-based pairwise judgments and human expert evaluations across three assessment dimensions.}
\label{tab:human_alignment}
\end{minipage}
\hfill
\begin{minipage}{0.5\textwidth}
\centering
\small
\begin{tabular}{lcc}
\toprule
 & \textbf{Accept} & \textbf{Reject} \\
\midrule
Rating & 2.547 & 2.043 \\
Contribution & 2.432 & 2.030 \\
Soundness & 1.853 & 1.542 \\
Novelty & 2.627 & 2.234 \\
\bottomrule
\end{tabular}
\caption{Evaluation of 50 NeurIPS 2025 papers using LigBench.}
\label{tab:nips2025_evaluation}
\end{minipage}

\end{table}

As shown in Table~\ref{tab:human_alignment}, the LLM exhibits substantial agreement with human experts across all three dimensions. In particular, the highest alignment is observed for \textit{contribution}, indicating that the model is effective at assessing the relative significance of research ideas within a domain. The consistently strong agreement on \textit{rating} and \textit{soundness} further suggests that the LLM’s pairwise judgments capture key aspects of overall quality and methodological rigor. These results provide additional evidence that LigBench produces evaluations that are well aligned with expert human preferences, complementing the automated assessments based on debiased review scores.

\subsection{Paper Evaluation}
We conducted an evaluation on papers from \textbf{NeurIPS 2025} to demonstrate the effectiveness of our framework. The evaluation applied the full LigBench procedure. The results of this evaluation are summarized in Table~\ref{tab:nips2025_evaluation}, highlighting the performance of the framework across the different assessment dimensions.


    

As shown in Table~\ref{tab:nips2025_evaluation}, papers that were accepted at NeurIPS 2025 consistently receive higher scores across all four evaluation dimensions compared to rejected papers. The largest differences are observed in \textit{novelty} and \textit{rating}, suggesting that both originality and overall perceived quality are strong determinants of acceptance. Differences in \textit{contribution} and \textit{soundness} are also notable, indicating that methodological rigor and domain impact are effectively captured by LigBench.

These results demonstrate that LigBench can meaningfully distinguish higher-quality research ideas from lower-quality ones, even when applied to a relatively small test set of 50 papers. This alignment with actual acceptance outcomes provides preliminary validation of the framework’s effectiveness in reflecting human and peer-review judgments.

\subsection{Impact of Data Source and Debiasing}

To examine the consistency of our evaluation framework across different data sources, we analyze the impact of the three major sources in the PAIR-IQ dataset: \textbf{ICLR 2024}, \textbf{ICLR 2025}, and \textbf{NeurIPS 2024}. Although these venues differ in review styles and score distributions, our debiasing procedure aims to normalize such discrepancies and produce a unified reference space.

To validate this, we conduct an ablation study where the score updating process is performed using only a single source at a time. Specifically, we independently use papers from ICLR 2024, ICLR 2025, and NeurIPS 2024 as the comparison pool for pairwise evaluation and iterative score updating. The differences between the scores obtained using a single source and those obtained using all sources together are summarized in Table~\ref{tab:data_source_analysis}.


\begin{table}[ht]
\centering
\begin{minipage}[t]{0.48\textwidth}
\centering
\begin{tabular}{lccc}
\toprule
Conference & Rating & Contrib. & Sound. \\
\midrule
ICLR 2024 & $\pm$3.2\% & $\pm$0.9\% & $\pm$1.5\% \\
ICLR 2025 & $\pm$2.0\% & $\pm$4.1\% & $\pm$3.4\% \\
NeurIPS 2024 & $\pm$2.3\% & $\pm$2.0\% & $\pm$2.6\% \\
\bottomrule
\end{tabular}
\caption{Impact of individual data sources compared to all sources. Percentage differences ($\pm$) between using a single data source versus all sources. Rows are conferences, columns are evaluation metrics.}
\label{tab:data_source_analysis}
\end{minipage}%
\hfill
\begin{minipage}[t]{0.48\textwidth}
\centering
\small
\begin{tabular}{lcccc}
\toprule
Model & Rating & Contrib. & Sound. & Novelty \\
\midrule
CoI & 3.468 & 2.996 & 2.487 & 2.474 \\
SciPIP & 2.621 & 1.900 & \textbf{3.243} & 2.267 \\
\midrule
GPT-5.2 & \textbf{3.974} & \textbf{3.392} & 3.022 & \textbf{2.583} \\
GPT-5 & 3.686 & 3.210 & 2.866 & 2.528 \\
GPT-4.1 & 1.000 & 0.760 & 1.877 & 2.194 \\
GPT-4o & 0.672 & 0.514 & 1.264 & 1.877 \\
\bottomrule
\end{tabular}
\caption{Evaluation of idea-generation frameworks and standalone LLMs using LigBench. Scores are averaged across 110 generated ideas per model. CoI and SciPIP both use GPT-5 as the backbone model.}
\label{tab:benchmark}
\end{minipage}
\end{table}

As shown in the table, the differences between scores computed from individual sources and those computed from all sources are minor, with all metrics showing deviations within a few percentage points. This indicates that our evaluation results are largely insensitive to the choice of data source. Consequently, the debiasing strategy effectively mitigates distributional differences across venues, enabling stable and comparable evaluation regardless of whether a single source or all sources are used.

\subsection{Benchmark and Model Evaluation}
To comprehensively assess the quality of LLM-generated research ideas, we evaluate both standalone LLMs and idea-generation frameworks using LigBench. Each system is prompted to freely generate 10 ideas across 11 research topics, including Reinforcement Learning, Image Generation, AI Safety, Agent, Computer Vision, Audio and Speech, Embodied Intelligence, Multimodal, AI for Science, Large Language Models, and Training. Final scores are computed as the average across all generated ideas.

We compare two representative idea-generation frameworks, CoI\citep{coi} and SciPIP\citep{scipip}, against direct LLM baselines. Both employ GPT-5 as their backbone, enabling a controlled comparison that isolates the effect of framework design. From the results in Table~\ref{tab:benchmark}, several key observations can be drawn regarding the performance of standalone LLMs and idea-generation frameworks:

First, among standalone LLMs, GPT-5.2 and GPT-5 clearly outperform GPT-4.1 and GPT-4o across all dimensions, with GPT-4o scoring below 1.0 on rating and contribution, indicating that earlier models struggle to meet basic quality thresholds.

Second, idea-generation frameworks do not consistently outperform direct prompting on strong backbone models. Despite using GPT-5, CoI and SciPIP often score lower than standalone GPT-5, suggesting that these frameworks may optimize weaker models but introduce constraints for capable ones. Notably, SciPIP achieves the highest Soundness score among all evaluated systems, slightly exceeding its underlying model GPT-5.

Finally, while novelty scores broadly follow expectations, absolute differences are small, highlighting that generating truly novel ideas remains challenging for current LLMs. For an intuitive overview of model performance across all metrics, see the radar visualization in Appendix~\ref{app:radar}.

\section{Conclusion}

We introduced \textbf{LigBench}, a unified and human-aligned benchmark for evaluating research ideas generated by large language models and agentic frameworks. By decomposing idea quality into four dimensions—\textit{rating}, \textit{contribution}, \textit{soundness}, and \textit{novelty}—LigBench enables consistent and fine-grained assessment across diverse sources.

To support reliable evaluation, we constructed the \textbf{PAIR-IQ} dataset with standardized, debiased scores, enabling effective pairwise judgment and robust score updating. Results show that LigBench produces stable evaluations aligned with expert preferences, even when individual pairwise judgments are imperfect.

LigBench provides a practical foundation for benchmarking future idea-generation systems and advancing automated evaluation of scientific creativity. It can also serve as a reliable reward signal for training idea-generation models via reinforcement learning, enabling more scalable and human-aligned AI-driven idea generation.


\bibliography{main}
\clearpage
\newpage

\appendix




\section{Hyper-parameters for Fine-Tuning}
\label{app:training_hyperparameter}


\begin{table}[htbp]
\centering
\caption{Hyper-parameters for Fine-Tuning.}
\begin{tabular}{l c}
    \toprule
    \textbf{Hyper-Parameter} & \textbf{Default Value} \\
    \midrule
    Finetuning Type & LoRA \\
    LoRA Target & all \\
    LoRA Rank & \num{16} \\
    LoRA Alpha & \num{16} \\
    LoRA Dropout & \num{0.05} \\
    \midrule
    Cutoff Length & \num{4096} \\
    \midrule
    Gradient Accumulation Steps & \num{16} \\
    Learning Rate & \num{1e-4} \\
    Train Epochs & \num{1.0} \\
    Learning Rate Scheduler & Cosine \\
    Warmup Ratio & \num{0.1} \\
    \bottomrule
\end{tabular}
\label{tab:training_hyperparameter}
\end{table}

\section{Prompt Specification}

\subsection{Heterogeneous Idea Formalization}
\label{app: formalization}
\begin{tcolorbox}[title={Heterogeneous Idea Formalization PROMPT}, colback=white,colframe=2,arc=1mm,boxrule=1pt,left=1mm,right=1mm,top=1mm,bottom=1mm]
\small
\color{2}
You are an expert agent specialized in research idea analysis and evaluation.

You will be provided with a raw research idea description. Your task is to formalize this idea into a structured representation that facilitates systematic comparison and evaluation.

Each research idea should be decomposed into the following four components:
	1.	Main Target:
A concise, one-sentence summary of the core task and research objective that the idea aims to address.
	2.	Core Breakthrough:
The key conceptual or technical advancement introduced by the idea, highlighting what fundamentally differentiates it from prior work.
	3.	Innovative Methods:
The concrete methodological approaches, algorithms, or techniques proposed to realize the main target.
	4.	Experimental Design:
The experimental setup and evaluation strategy used to validate the proposed methods, including datasets, benchmarks, baselines, and evaluation metrics when applicable.

Now, the new idea is:
<{new\_idea}>

Please formalize it into the four parts.

Output Format: A list of four strings in the following order:
[
  "<main\_target>",
  "<core\_breakthrough>",
  "<innovative\_method\_1>; <innovative\_method\_2>; <innovative\_method\_3>; ...",
  "<experimental\_design>"
]

Example:
[
  "To improve the robustness of large language models against distribution shifts",
  "Introducing a unified robustness-aware training objective that jointly models uncertainty and domain variance",
  "Uncertainty-guided representation learning; Domain-adaptive regularization; Robust contrastive pretraining",
  "Evaluation on multiple out-of-distribution NLP benchmarks with comparisons to standard fine-tuning and domain adaptation baselines"
]

\end{tcolorbox}

\subsection{Paper Idea Formalization}
\label{app:paper_formalization}
When constructing the PAIR-IQ dataset, we need to extract standardized research ideas from existing papers, which differs from directly normalizing heterogeneous, free-form ideas. Specifically, we first parse the full textual content of each paper and use an LLM to map diverse and heterogeneous section headings into five canonical categories: \textit{Abstract}, \textit{Introduction}, \textit{Related Work}, \textit{Method}, and \textit{Experiment}. Based on this unified structure, we perform targeted extraction for different components of an idea from the corresponding sections. The \textit{Main Target} is primarily extracted from the \textit{Abstract}. The \textit{Core Breakthrough} is derived from the \textit{Abstract}, \textit{Introduction}, and \textit{Related Work}. The \textit{Innovative Methods} are extracted from the \textit{Abstract} and \textit{Method}. The \textit{Experimental Design} is obtained from the \textit{Experiment} section. The detailed prompts used for this process are provided below.

\begin{tcolorbox}[title={Section Mapping PROMPT}, colback=white,colframe=3,arc=1mm,boxrule=1pt,left=1mm,right=1mm,top=1mm,bottom=1mm]
\small
\color{3}
You will get a list of section names from a research paper. Please identify which section(s) correspond to each of the following: abstract, introduction, related work, method, and experiment.

Note: Sometimes a part (e.g., method) may be split across multiple sections. In such cases, please use a list to record all matching section names, e.g., "method": ["Methodology", "Approach"].
Note: The "related work" section may also describe existing problems in the field, not just previous approaches.
The answer format MUST be a JSON object with keys "abstract", "introduction", "related\_work", "method", and "experiment". Each value should be either a list of section names (if present) or null (if not present). For example:
\{\{
    "abstract": ["Abstract"],
    "introduction": ["Introduction"],
    "related\_work": ["Related Work"],
    "method": ["Methodology", "Approach"],
    "experiment": ["Experiment"]
\}\}
Now, the section name list is: ```\{section\_names\}'''.

\end{tcolorbox}

\begin{tcolorbox}[title={Main Target PROMPT}, colback=white,colframe=3,arc=1mm,boxrule=1pt,left=1mm,right=1mm,top=1mm,bottom=1mm]
\small
\color{3}
Here you will get the abstract part of a research paper. You need to return the main target of this paper in one sentence, such as "This research designs a new model for image classification." or "This research proposes a new dataset for object detection" or "This research introduces a new framework for idea generation" and so on.

The answer format MUST be a simple string without any explanation or extra content.

Now , the abstract is: ```\{abstract\}'''.

\end{tcolorbox}

\begin{tcolorbox}[title={Core Breakthrough PROMPT}, colback=white,colframe=3,arc=1mm,boxrule=1pt,left=1mm,right=1mm,top=1mm,bottom=1mm]
\small
\color{3}
Here you will get three sections of a research paper: abstract, introduction, and related work. You need to return the core breakthrough of this paper in one sentence, such as what existing problems are solved or what advances are made compared to previous work.
    The answer format MUST be a simple string without any explanation or extra content.
    Now, the abstract part is: ```\{abstract\}'''.
    The introduction part is: ```\{introduction\}'''.
    The related work part is: ```\{related\_work\}'''.

\end{tcolorbox}

\begin{tcolorbox}[title={Innovative Methods PROMPT}, colback=white,colframe=3,arc=1mm,boxrule=1pt,left=1mm,right=1mm,top=1mm,bottom=1mm]
\small
\color{3}
Here you will get the abstract and method sections of a research paper. Please list three to four of the most innovative and core methods proposed in this paper, each as a separate numbered item. Keep your answer concise and focus on the key elements only.
    The answer format MUST be a simple numbered list (e.g., "1. ... 2. ...") without any explanation or extra content.
    Now, the abstract part is: ```\{abstract\}'''.
    The method part is: ```\{method\}'''.

\end{tcolorbox}

\begin{tcolorbox}[title={Experimental Design PROMPT}, colback=white,colframe=3,arc=1mm,boxrule=1pt,left=1mm,right=1mm,top=1mm,bottom=1mm]
\small
\color{3}
Here you will get the abstract and experiment sections of a research paper. Please list the main experimental designs proposed in this paper, each as a separate numbered item. Ignore any results, numbers, or metrics mentioned in the experiment section; only describe what experiments were designed.
    The answer format MUST be a simple numbered list (e.g., "1. ... 2. ...") without any explanation or extra content.
    Now, the abstract part is: ```\{abstract\}'''.
    The experiment part is: ```\{experiment\}'''.

\end{tcolorbox}

\subsection{Pairwise Comparison}
\label{pair com}

\begin{tcolorbox}[title={Initial Scoring with Examples PROMPT}, colback=white,colframe=2,arc=1mm,boxrule=1pt,left=1mm,right=1mm,top=1mm,bottom=1mm]
\small
\color{2}

You are an intelligent agent who is expert in scientific hypothesis evaluation.

A research idea can be decomposed into four part: Main Target, Core Breakthrough, Innovative Approach, and Experimental Design.

Here you will get a new idea and three published papers' ideas. You need to score the new idea in four dimensions: Rating, contribution, Soundness, and Novelty.

Dimension Definitions:

- Rating: The overall quality of the idea and a comprehensive evaluation

- Contribution: Whether the outcomes will be useful for future work

- Soundness: The logical and methodological integrity of the research

- Novelty: The originality and potential to contribution new insights or approaches

We have scored these three published papers in four dimensions. The "comments" record the reasons for the scores.

Please read these three published papers' scores and comments carefully, and then score the new idea.

New Idea:
```
Main Target: \{new\_idea\_main\}
Core Breakthrough: \{new\_idea\_breakthrough\}
Innovative Approach: \{new\_idea\_approach\}
Experimental Design: \{new\_idea\_experiment\}

'''

Calibration Example 1:
```
\{example1\}
'''

Calibration Example 2:
```
\{example2\}
'''

Calibration Example 3:
```
\{example3\}
'''

Important: Think carefully, as a new idea may appear highly novel and compelling but lack sufficient supporting details.

Output Format: A JSON list of four floating-point numbers (0-5):
[<rating>, <contribution>, <soundness>, <novelty>]

Example: [3.7, 4.1, 2.6, 2.9]

\end{tcolorbox}









   
   
   
   


   


   
   

   

    
    
    


\begin{tcolorbox}[title={Pairwise Comparison PROMPT}, colback=white,colframe=2,arc=1mm,boxrule=1pt,left=1mm,right=1mm,top=1mm,bottom=1mm]
\small
\color{2}
You are an intelligent agent who is expert in **scientific hypothesis evaluation**.  
A research idea can be decomposed into four part: Main Target, Core Breakthrough, Innovative Approach, and Experimental Design.

Here you will get two research ideas; you need to compare them on three dimensions: Rating, Soundness and Contribute, to judge which one is better on each dimension.  
Dimensions definitions:
- Rating: The overall quality of the idea. This is a comprehensive evaluation that summarizes the strengths and weaknesses of the idea as a whole.
- Soundness: The logical and methodological integrity of the research. Assess whether the idea is based on solid reasoning, uses appropriate methods, and avoids major flaws or unsupported assumptions.
- Contribute: Whether the outcomes will be useful for future work. Evaluate if the idea can advance scientific knowledge, provide valuable insights, or serve as a foundation for further research in the field.

Your answer format should be a think process between <think> and </think>, containing three steps: `1. ... 2. ... 3. ...', where 1 is the reasoning for Rating, 2 is the reasoning for Soundness, and 3 is the reasoning for Contribution, and a final list of three numbers (0, 1, or 0.5) in the order [Rating, Soundness, Contribution], where
- `1' means the first idea is better,
- `0.5' means the two ideas are equal,
- `0' means the second idea is better.
For example, if the first idea wins on Rating, loses on Soundness, and is equal on Contribution, your answer should be: `<think>...</think>[1, 0, 0.5]'
When considering the rating, you may use the Main Target and Innovative Approach as the most important basis.
When considering the soundness, you may use the Main Target, Innovative Approach and Experimental Design as the most important basis.
When considering the contribution, you may use the Main Target and Core Breakthrough as the most important basis.

Now:
First idea:
```The Main Target is \{idea\_A\_main\}.  
And the Core Breakthrough is \{idea\_A\_breakthrough\}.
The Innovative Approach is \{idea\_A\_approach\}.
The Experimental Design is \{idea\_A\_experiment\}.'''

Second idea:  
```The Main Target is \{idea\_B\_main\}.  
And the Core Breakthrough is \{idea\_B\_breakthrough\}.
The Innovative Approach is \{idea\_B\_approach\}.
The Experimental Design is \{idea\_B\_experiment\}.'''

Please give your answer directly according to the above format.

\end{tcolorbox}

\section{PAIR-IQ extra information}
\label{app:PAIR-IQ extra information}
Figure~\ref{fig:category_dist} illustrates the dataset statistics of PAIR-IQ, showing the distribution of collected papers across different conference venues.

An example of the standardized paper idea representation used in the PAIR-IQ dataset is shown below.

\begin{lstlisting}[language=json, caption={Example of a Paper Idea in PAIR-IQ}, basicstyle=\scriptsize\ttfamily,label={lst:Example of a Paper Idea in PAIR-IQ}]
{
  "title": "A Generative Model of Symmetry Transformations",
  "url": "https://openreview.net/forum
          ?id=aFP24eYpWh",
  "conference": "NeurIPS",
  "year": 2024,
  "label": "accept-poster",
  "abstract": "Correctly capturing the symmetry transformations of data can lead to efficient models with strong generalization capabilities ...",
  "keywords": "approximate symmetries, invariances, deep generative models",
  "primary area": "Generative models",
  "rating": 2.9070274219833774,
  "contribute": 2.587674461477269,
  "soundness": 2.843680907598996,
  "novelty": "0",
  "analysis": {
    "main_target": "This research constructs a generative model that learns to capture the approximate symmetries of data from a prespecified set of possible symmetries.",
    "core_breakthrough": "This paper introduces a Symmetry-aware Generative Model (SGM) that learns to capture the approximate symmetries present in data, leading to improved data efficiency and higher marginal test-log-likelihoods when combined with standard generative models.",
    "innovative_methods": "1. A generative model that decomposes the data distribution into a distribution over prototypes and a distribution over transformation parameters.\n2. A self-supervised learning scheme to train an equivariant transformation inference function.\n3. A method for learning the distribution over transformation parameters using maximum likelihood, allowing the model to capture the extent of symmetries present in the data.\n4. Composing transformations in a transformation parameter space to minimize the number of applications and reduce discretization errors.",
    "experimental_design": "1. Exploring the SGM’s ability to learn symmetries by producing valid prototypes and generating plausible samples from the data distribution given those prototypes.\n2. Leveraging the SGM to improve data efficiency in deep generative models.\n3. Conducting experiments on dSprites, MNIST, and GalaxyMNIST..."
  },
  "embedding": [1024]
}
\end{lstlisting}

\begin{figure}[t]
  \centering
  \includegraphics[width=0.48\textwidth,trim=0 0 0 0,clip]{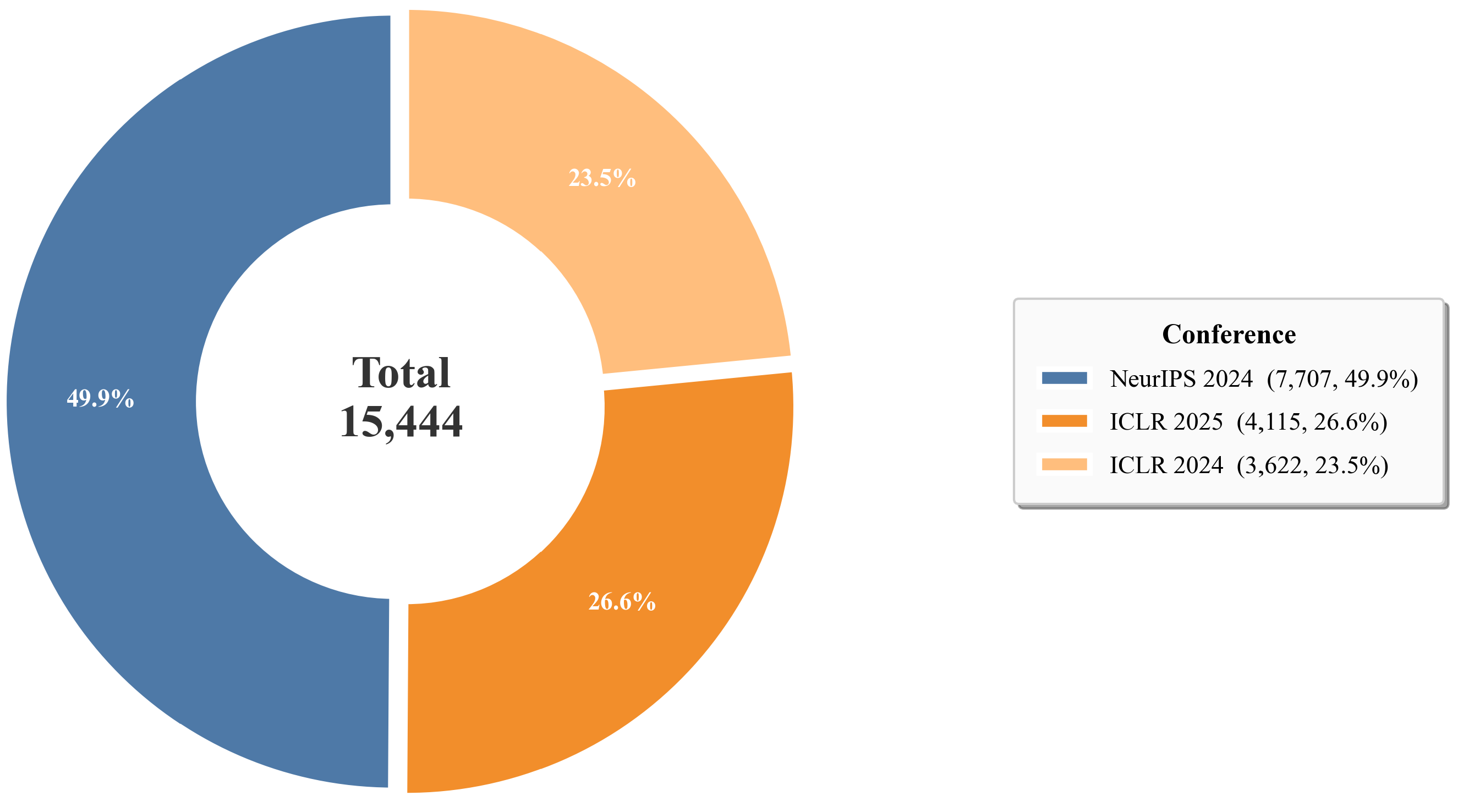} \hfill
  \caption {Distribution by conference venues.}
  \label{fig:category_dist}
\end{figure}

\section{LigBench Evaluation Radar Chart}
\label{app:radar}

Figure~\ref{fig:radar} visualizes the performance of different idea-generation frameworks and standalone LLMs summarized in Table~\ref{tab:benchmark}. Each axis corresponds to one of the evaluation metrics (Rating, Contrib., Sound., Novelty), allowing a quick comparison of the models' relative strengths. GPT-5.2 consistently achieves the highest overall scores, while CoI shows competitive results despite relying on GPT-5 as its backbone. This visualization complements the numeric results and provides an intuitive overview of model performance.

\begin{figure}[t]
    \centering
    \includegraphics[width=0.7\linewidth]{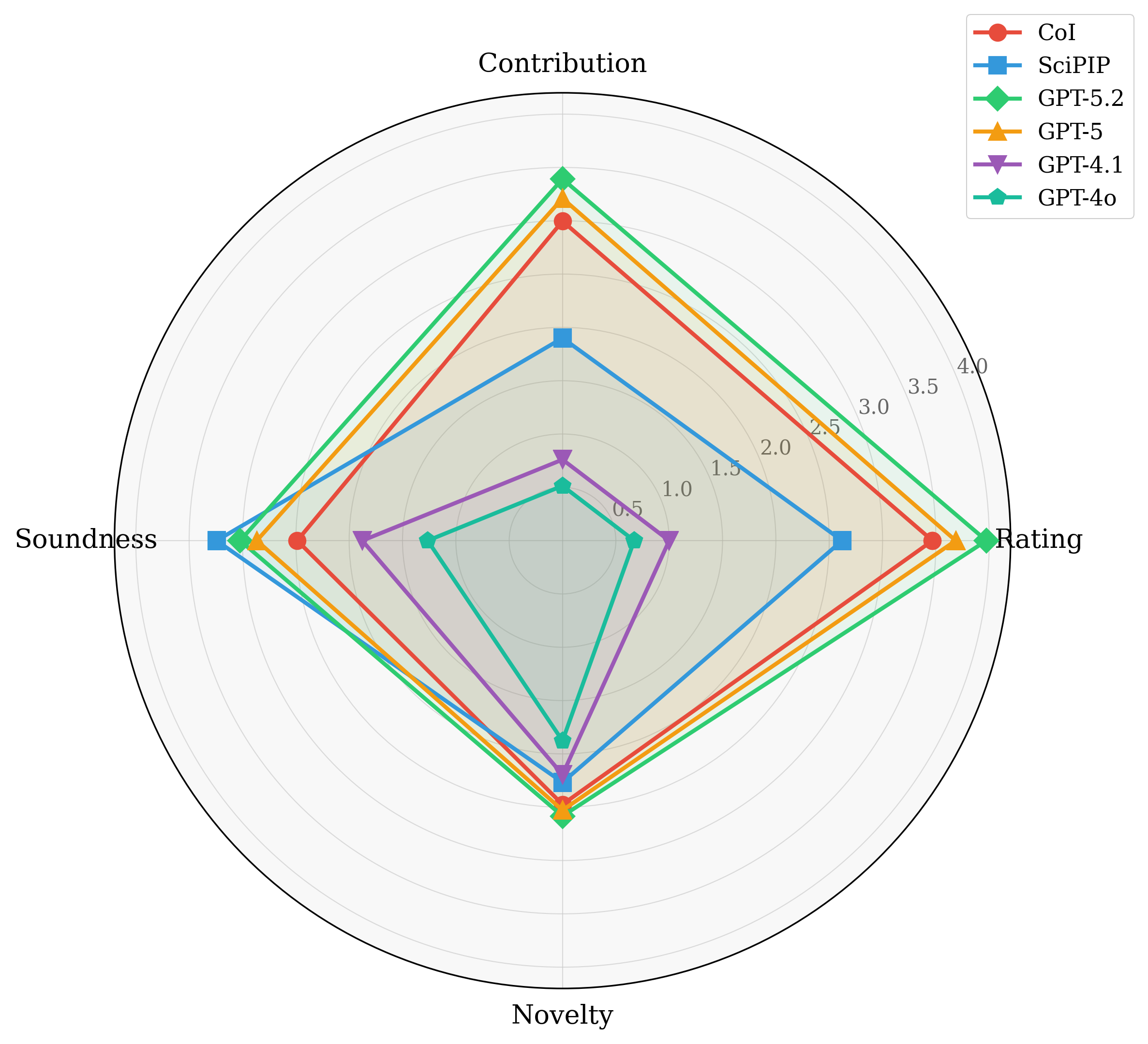}
    \caption{Radar chart visualization of Table~\ref{tab:benchmark}. GPT-5.2 achieves the highest overall performance, while CoI demonstrates competitive results despite using GPT-5 as its backbone.}
    \label{fig:radar}
\end{figure}

\section{Elo-based Score Update: Mathematical Analysis}
\label{appendix:elo}

This appendix provides a comprehensive mathematical treatment of the Elo-based score update mechanism used in LigBench, including theoretical foundations, convergence analysis, and empirical parameter optimization.

\subsection{Algorithm Overview}

The pairwise comparison module employs an adapted Elo rating system to iteratively update the scores of target ideas. The complete procedure is formalized in Algorithm~\ref{alg:elo}.

\begin{algorithm}[h]
\caption{Elo-based Pairwise Score Update}
\label{alg:elo}
\begin{algorithmic}[1]
\REQUIRE Target idea $I$, Reference paper set $\mathcal{R}$, Initial score $s_0$, Convergence threshold $\epsilon$, Maximum iterations $T_{\max}$
\ENSURE Final score $s^*$
\STATE $s \leftarrow s_0$; $t \leftarrow 0$
\WHILE{$t < T_{\max}$}
    \STATE Sample reference paper $R \sim \mathcal{R}$
    \STATE Compute expected outcome: $E = \frac{1}{1 + 10^{(s_R - s) / d}}$
    \STATE Obtain LLM judgment: $S \in \{0, 0.5, 1\}$
    \STATE Compute adaptive K-factor: $K_t = K_{\max} \cdot \gamma^t$
    \STATE Update raw score: $s^{\text{raw}} = s + K_t \cdot (S - E)$
    \STATE Apply soft clamping: $s^{\text{new}} = \mathcal{C}(s^{\text{raw}})$
    \IF{$|s^{\text{new}} - s| < \epsilon$}
        \STATE \textbf{break}
    \ENDIF
    \STATE $s \leftarrow s^{\text{new}}$; $t \leftarrow t + 1$
\ENDWHILE
\RETURN $s^* = s$
\end{algorithmic}
\end{algorithm}

\subsection{Theoretical Foundation}

\subsubsection{Expected Outcome Function}

The expected outcome function is derived from the Bradley-Terry model \citep{bradley1952rank}, which assumes that the probability of item $A$ being preferred over item $B$ follows a logistic function of their latent quality difference:
\begin{equation}
    P(A \succ B) = \frac{1}{1 + 10^{(s_B - s_A) / d}}.
\end{equation}

This formulation satisfies several desirable properties:
\begin{itemize}
    \item \textbf{Symmetry}: $P(A \succ B) + P(B \succ A) = 1$
    \item \textbf{Transitivity}: If $s_A > s_B > s_C$, then $P(A \succ C) > P(A \succ B)$
    \item \textbf{Monotonicity}: $\frac{\partial P(A \succ B)}{\partial s_A} > 0$
\end{itemize}

\subsubsection{Score Update as Gradient Descent}

The score update rule can be interpreted as a stochastic gradient descent step on a cross-entropy loss function. Let $\mathcal{L}(s_A)$ denote the expected loss:
\begin{equation}
    \mathcal{L}(s_A) = -\mathbb{E}_{S_A}[S_A \log E_A + (1 - S_A) \log(1 - E_A)].
\end{equation}

Taking the gradient with respect to $s_A$:
\begin{equation}
    \frac{\partial \mathcal{L}}{\partial s_A} = \frac{\ln 10}{d} \cdot (E_A - S_A).
\end{equation}

Thus, the update rule $s_A \leftarrow s_A + K(S_A - E_A)$ corresponds to gradient descent with learning rate $\eta = \frac{Kd}{\ln 10}$.

\subsection{Parameter Selection}

\subsubsection{Scaling Parameter $d$}

In standard chess Elo systems, $d = 400$ with ratings spanning approximately $[0, 3000]$. Our framework operates on a normalized score range of $[0, 5]$. To preserve the relative sensitivity of the expected outcome function, we rescale $d$ proportionally:
\begin{equation}
    d_{\text{new}} = d_{\text{chess}} \cdot \frac{\Delta s_{\text{new}}}{\Delta s_{\text{chess}}} = 400 \cdot \frac{5}{3000} \approx 0.67.
\end{equation}

However, empirical experiments revealed that this theoretical value leads to overly aggressive score changes. We conducted a grid search over $d \in \{0.5, 0.75, 1.0, 1.25, 1.5, 1.75, 2.0\}$ using a validation set of 500 idea pairs with known ground-truth scores.

\begin{table}[h]
\centering
\small
\begin{tabular}{cccc}
\toprule
$d$ & MAE $\downarrow$ & Rank Corr. $\uparrow$ & Conver. Rate $\uparrow$ \\
\midrule
0.50 & 0.412 & 0.723 & 0.89 \\
0.75 & 0.356 & 0.761 & 0.92 \\
1.00 & 0.298 & 0.789 & 0.94 \\
1.25 & 0.267 & 0.812 & 0.96 \\
\textbf{1.50} & \textbf{0.251} & \textbf{0.831} & \textbf{0.97} \\
1.75 & 0.263 & 0.819 & 0.95 \\
2.00 & 0.289 & 0.798 & 0.93 \\
\bottomrule
\end{tabular}
\caption{Parameter sensitivity analysis for scaling parameter $d$.}
\label{tab:d_sensitivity}
\end{table}

Based on these results, we set $d = 1.5$, which achieves the optimal balance between score discrimination sensitivity and convergence stability.

\subsubsection{Adaptive K-factor Strategy}

The K-factor controls the magnitude of score adjustments. We employ an exponential decay strategy:
\begin{equation}
    K_t = K_{\max} \cdot \gamma^t,
\end{equation}
where $K_{\max}$ is the initial K-factor, $\gamma \in (0, 1)$ is the decay rate, and $t$ is the iteration number.

This adaptive strategy addresses two competing objectives: (1) \textbf{Fast initial convergence}: Large $K$ values enable rapid movement toward the true score when the initial estimate is poor. (2) \textbf{Fine-grained calibration}: Small $K$ values prevent oscillation around the equilibrium and enable precise final positioning.

We optimized $(K_{\max}, \gamma)$ through grid search:

\begin{table}[h]
\centering
\small
\begin{tabular}{ccccc}
\toprule
$K_{\max}$ & $\gamma$ & Iterations $\downarrow$ & Final MAE $\downarrow$ & Stability $\uparrow$ \\
\midrule
0.3 & 0.90 & 28.4 & 0.287 & 0.91 \\
0.3 & 0.95 & 35.2 & 0.264 & 0.94 \\
0.5 & 0.90 & 18.7 & 0.271 & 0.89 \\
\textbf{0.5} & \textbf{0.95} & \textbf{22.3} & \textbf{0.251} & \textbf{0.97} \\
0.5 & 0.98 & 31.6 & 0.249 & 0.96 \\
0.7 & 0.90 & 14.2 & 0.298 & 0.82 \\
0.7 & 0.95 & 17.8 & 0.273 & 0.88 \\
\bottomrule
\end{tabular}
\caption{Grid search results for K-factor parameters.}
\label{tab:k_sensitivity}
\end{table}

The optimal configuration is $K_{\max} = 0.5$ and $\gamma = 0.95$, balancing convergence speed with final accuracy.

\subsection{Soft Clamping Mechanism}

Since our scores are constrained to $[0, 5]$, we apply a smooth boundary function using the hyperbolic tangent:
\begin{equation}
    \mathcal{C}(s) = 2.5 + 2.5 \cdot \tanh\left(\frac{s - 2.5}{2.5}\right).
\end{equation}

\subsubsection{Properties of the Clamping Function}

\begin{proposition}
The clamping function $\mathcal{C}: \mathbb{R} \to (0, 5)$ satisfies:
\begin{enumerate}
    \item $\mathcal{C}(2.5) = 2.5$ (fixed point at center)
    \item $\lim_{s \to \infty} \mathcal{C}(s) = 5$ and $\lim_{s \to -\infty} \mathcal{C}(s) = 0$
    \item $\mathcal{C}'(s) > 0$ for all $s \in \mathbb{R}$ (strictly monotonic)
    \item $\mathcal{C}'(2.5) = 1$ (identity mapping near center)
\end{enumerate}
\end{proposition}

\begin{proof}
Property 1: $\mathcal{C}(2.5) = 2.5 + 2.5 \cdot \tanh(0) = 2.5$.

Property 2: Since $\lim_{x \to \pm\infty} \tanh(x) = \pm 1$, we have $\lim_{s \to \infty} \mathcal{C}(s) = 2.5 + 2.5 = 5$ and $\lim_{s \to -\infty} \mathcal{C}(s) = 2.5 - 2.5 = 0$.

Property 3: $\mathcal{C}'(s) = \text{sech}^2\left(\frac{s-2.5}{2.5}\right) > 0$.

Property 4: $\mathcal{C}'(2.5) = \text{sech}^2(0) = 1$.
\end{proof}

The soft clamping preserves the relative ordering of scores while ensuring boundary constraints are satisfied smoothly, unlike hard clipping which introduces discontinuities in the gradient.




The tanh-based clamping was selected because it provides symmetric behavior around the midpoint ($s = 2.5$), maintains near-identity mapping for scores in the typical range $[1, 4]$, and gracefully compresses extreme values.

\subsection{Convergence Analysis}

\subsubsection{Convergence Conditions}

\begin{theorem}[Convergence of Elo Updates]
\label{thm:convergence}
Let $\{s_t\}_{t=0}^{\infty}$ be the sequence of scores generated by the Elo update rule with adaptive K-factor $K_t = K_{\max} \cdot \gamma^t$ where $\gamma \in (0, 1)$. Assume the LLM judgment is consistent with the ground-truth ordering with probability $p > 0.5$. Then $\{s_t\}$ converges almost surely to a fixed point $s^*$.
\end{theorem}

\begin{proof}[Proof Sketch]
The proof relies on three key observations:

\textbf{Step 1 (Bounded Updates):} The update magnitude is bounded by $|K_t(S_A - E_A)| \leq K_t$ since $S_A, E_A \in [0, 1]$. Combined with the soft clamping, scores remain in $[0, 5]$.

\textbf{Step 2 (Summable Learning Rates):} The series $\sum_{t=0}^{\infty} K_t = K_{\max} \sum_{t=0}^{\infty} \gamma^t = \frac{K_{\max}}{1-\gamma} < \infty$, ensuring diminishing step sizes.

\textbf{Step 3 (Martingale Convergence):} Define the expected score change $\Delta_t = \mathbb{E}[s_{t+1} - s_t | s_t]$. Under the consistency assumption, $\Delta_t$ has the same sign as $(s^* - s_t)$ where $s^*$ is the true score. By the martingale convergence theorem, $s_t \to s^*$ almost surely.
\end{proof}

\subsubsection{Stability Under Noisy Judgments}

In practice, LLM judgments contain noise. We analyze robustness by introducing artificial noise:

\begin{table}[h]
\centering
\resizebox{0.5\textwidth}{!}{

\begin{tabular}{cccc}
\toprule
\makecell{Noise\\Level} 
& \makecell{Final\\MAE} 
& \makecell{Rank\\Correlation} 
& \makecell{Convergence\\(\%)} \\
\midrule
0\% (ideal) & 0.251 & 0.831 & 100\% \\
10\% & 0.267 & 0.812 & 99.2\% \\
20\% & 0.298 & 0.784 & 97.6\% \\
30\% & 0.342 & 0.751 & 94.1\% \\
\bottomrule
\end{tabular}
}
\caption{Robustness to judgment noise.}
\label{tab:noise}
\end{table}

The algorithm maintains reasonable performance even with 20-30\% judgment error, demonstrating robustness to imperfect LLM evaluations.

\section{Novelty Assessment: Detailed Methodology}
\label{appendix:novelty}

This appendix provides a comprehensive description of the novelty assessment module, including the mathematical formulation, parameter selection rationale, and empirical validation.

\subsection{Three-Stage Evaluation Pipeline}

The novelty assessment proceeds through three sequential stages, each contributing a distinct aspect to the final evaluation.

\subsubsection{Stage 1: LLM-based Initialization}

We first obtain an initial novelty estimate $s_{\text{novelty}}^{(0)}$ from the LLM evaluator. This provides a baseline assessment based on the model's understanding of the idea's originality and potential to introduce new insights. The LLM is prompted to rate novelty on a 0-5 scale, considering whether the approach incorporates eye-catching aspects or methods not commonly seen in the field.

The initialization prompt instructs the model to evaluate: (1) whether the core idea represents a departure from existing approaches; (2) the presence of unconventional methodological choices; and (3) the potential to open new research directions.

\subsubsection{Stage 2: Similarity-Based Quantification}

To objectively measure novelty against the current research landscape, we retrieve semantically related papers using the Semantic Scholar API. Based on the formalized representation of the target idea, we search for the most relevant publications and compute their cosine similarities $\{s_1, s_2, \ldots, s_N\}$ with the target idea using pre-trained embedding models.

\paragraph{Weighted Similarity Metric.}
We construct a weighted similarity metric that balances maximum and average similarity:
\begin{equation}
    s_{\text{weighted}} = \alpha \cdot \max(\{s_i\}) + (1-\alpha) \cdot \frac{1}{N}\sum_{i=1}^{N} s_i,
    \label{eq:weighted_sim}
\end{equation}
where $\alpha$ controls the emphasis on the most similar existing work.

\paragraph{Rationale for Weighting.}
The weighting scheme reflects the intuition that novelty is fundamentally constrained by the most similar existing work: even if an idea differs from most papers, high similarity to a single prior publication significantly limits its novelty. Setting $\alpha = 0.6$ emphasizes this constraint while still accounting for the broader similarity landscape.

\paragraph{Inverse Sigmoid Mapping.}
The weighted similarity is mapped to a novelty score through an inverse sigmoid function:
\begin{equation}
    s_{\text{sim}} = 5 \cdot \left(1 - \frac{1}{1 + e^{-k(s_{\text{weighted}} - \mu)}}\right),
    \label{eq:sigmoid_map}
\end{equation}
where $k$ controls the steepness of the transition and $\mu$ defines the inflection point. This continuous mapping ensures that high similarity to existing work leads to low novelty scores, while low similarity indicates high novelty.

\subsubsection{Stage 3: Continuous Fusion}

The final novelty score integrates both the LLM's subjective assessment and the objective similarity-based measurement:
\begin{equation}
    s_{\text{novelty}}^{\text{final}} = \beta \cdot s_{\text{sim}} + (1-\beta) \cdot s_{\text{novelty}}^{(0)},
    \label{eq:fusion}
\end{equation}
where $\beta$ determines the relative contribution of each component.

\subsection{Parameter Selection}

\subsubsection{Weighting Parameter $\alpha$}

We conducted experiments to determine the optimal value of $\alpha$ in Equation~\ref{eq:weighted_sim}. The results are shown in Table~\ref{tab:alpha_sens}.

\begin{table}[!t]
\centering
\small
\caption{Effect of weighting parameter $\alpha$ on novelty assessment quality.}
\label{tab:alpha_sens}
\begin{tabular}{@{}lccc@{}}
\toprule
$\alpha$ & Hum.Corr.$\uparrow$ & Discrim.$\uparrow$ & Stab.$\uparrow$ \\
\midrule
0.0 (avg only) & 0.612 & 0.534 & 0.91 \\
0.3 & 0.658 & 0.587 & 0.93 \\
0.5 & 0.691 & 0.623 & 0.94 \\
\textbf{0.6} & \textbf{0.724} & \textbf{0.651} & \textbf{0.95} \\
0.7 & 0.718 & 0.642 & 0.94 \\
0.8 & 0.687 & 0.598 & 0.92 \\
1.0 (max only) & 0.621 & 0.512 & 0.89 \\
\bottomrule
\end{tabular}
\end{table}

The results show that $\alpha = 0.6$ achieves the best correlation with human novelty judgments while maintaining stable assessments.

\subsubsection{Sigmoid Parameters $k$ and $\mu$}

The parameters $k$ and $\mu$ in Equation~\ref{eq:sigmoid_map} control the shape of the similarity-to-novelty mapping.

\paragraph{Inflection Point $\mu$.}
The inflection point $\mu = 0.72$ was calibrated based on empirical analysis of similarity distributions in the PAIR-IQ dataset. At this similarity level, ideas typically exhibit moderate overlap with existing work, justifying a neutral novelty score of 2.5. Table~\ref{tab:sim_stats} shows the similarity distribution statistics.

\begin{table}[!t]
\centering
\small
\caption{Similarity distribution statistics from PAIR-IQ.}
\label{tab:sim_stats}
\begin{tabular}{@{}lc@{}}
\toprule
Statistic & Value \\
\midrule
Mean similarity & 0.68 \\
Median similarity & 0.71 \\
Standard deviation & 0.12 \\
25th percentile & 0.61 \\
75th percentile & 0.79 \\
\bottomrule
\end{tabular}
\end{table}

\paragraph{Steepness Parameter $k$.}
The steepness parameter $k = 12$ was selected to ensure meaningful score differentiation across the typical similarity range $[0.5, 0.9]$. Table~\ref{tab:k_sens} shows the effect of different $k$ values.

\begin{table}[!t]
\centering
\small
\caption{Effect of steepness parameter $k$ on score distribution.}
\label{tab:k_sens}
\begin{tabular}{@{}lccc@{}}
\toprule
$k$ & Range$\uparrow$ & Entropy$\uparrow$ & Hum.Corr.$\uparrow$ \\
\midrule
6 & 2.1 & 1.82 & 0.687 \\
9 & 3.2 & 1.91 & 0.712 \\
\textbf{12} & \textbf{4.1} & \textbf{1.98} & \textbf{0.724} \\
15 & 4.5 & 1.94 & 0.718 \\
18 & 4.7 & 1.87 & 0.703 \\
\bottomrule
\end{tabular}
\end{table}

\subsubsection{Fusion Weight $\beta$}

The fusion weight $\beta = 0.7$ was determined through ablation studies comparing different weighting schemes, as shown in Table~\ref{tab:beta_abl}.

\begin{table}[!t]
\centering
\small
\caption{Ablation study on fusion weight $\beta$.}
\label{tab:beta_abl}
\begin{tabular}{@{}lccc@{}}
\toprule
$\beta$ & Hum.Corr.$\uparrow$ & Robust.$\uparrow$ & Interp. \\
\midrule
0.0 (LLM only) & 0.651 & 0.72 & High \\
0.3 & 0.689 & 0.81 & High \\
0.5 & 0.712 & 0.87 & Medium \\
\textbf{0.7} & \textbf{0.724} & \textbf{0.93} & Medium \\
0.9 & 0.718 & 0.95 & Low \\
1.0 (Sim only) & 0.698 & 0.96 & Low \\
\bottomrule
\end{tabular}
\end{table}

Setting $\beta = 0.7$ prioritizes the objective similarity-based measurement while retaining sufficient LLM contribution to capture semantic nuances not reflected in embedding similarity.

\subsection{Mathematical Properties}

\subsubsection{Properties of the Inverse Sigmoid Mapping}

The mapping function defined by Equation~\ref{eq:sigmoid_map} satisfies several desirable properties.

\begin{proposition}
\label{prop:sigmoid}
The mapping function $f: [0,1] \to [0,5]$ satisfies:
\begin{enumerate}
    \item $f$ is strictly monotonically decreasing
    \item $f(\mu) = 2.5$ (midpoint property)
    \item $\lim_{s \to 0} f(s) = 5$ and $\lim_{s \to 1} f(s) = 0$
    \item $f$ is smooth ($C^\infty$) on $(0,1)$
\end{enumerate}
\end{proposition}

\begin{proof}
Property 1: Taking the derivative,
\begin{equation}
    f'(s) = -5 \cdot \frac{k \cdot e^{-k(s-\mu)}}{(1 + e^{-k(s-\mu)})^2} < 0
\end{equation}
for all $s \in (0,1)$, confirming strict monotonic decrease.

Property 2: At $s = \mu$,
\begin{equation}
    f(\mu) = 5 \cdot \left(1 - \frac{1}{1 + e^0}\right) = 5 \cdot \left(1 - \frac{1}{2}\right) = 2.5.
\end{equation}

Property 3: As $s \to 0$, $e^{-k(s-\mu)} \to e^{k\mu} \to \infty$, so $f(s) \to 5$. As $s \to 1$, $e^{-k(s-\mu)} \to 0$, so $f(s) \to 0$.

Property 4: $f$ is a composition of smooth functions (exponential and rational), hence smooth on the interior.
\end{proof}

\subsubsection{Sensitivity Analysis}

The sensitivity of the novelty score to changes in weighted similarity is given by:
\begin{equation}
    \left|\frac{\partial s_{\text{novelty}}^{\text{final}}}{\partial s_{\text{weighted}}}\right| = \beta \cdot \frac{5k \cdot e^{-k(s_{\text{weighted}}-\mu)}}{(1 + e^{-k(s_{\text{weighted}}-\mu)})^2}.
\end{equation}

This sensitivity is maximized at $s_{\text{weighted}} = \mu$, where small changes in similarity lead to the largest score adjustments, and diminishes toward the extremes, providing natural robustness to outliers.

\end{document}